\documentclass[preprint,12pt,authoryear]{elsarticle}

\usepackage{amssymb}

\journal{Computers and Operations Research}

\usepackage{indentfirst}
\usepackage{setspace}
\usepackage{soul}
\usepackage{float}
\usepackage[colorlinks=true, linkcolor=black, urlcolor=blue, citecolor=blue, anchorcolor=blue]{hyperref}
\usepackage{wrapfig}
\usepackage{comment}
\usepackage{textgreek}
\usepackage{url}
\usepackage[font=footnotesize,labelfont=it]{caption}

\usepackage{enumitem, comment, xifthen}
\usepackage{amsmath,amssymb,amsthm, amsfonts}
\usepackage[T1]{fontenc}
\usepackage{makecell}
\renewcommand*{\bibfont}{\small}

\usepackage{booktabs}
\usepackage{hhline}
\usepackage{array,multirow}
\usepackage{lipsum}
\usepackage{siunitx,etoolbox}

\usepackage{graphicx}

\usepackage{tikz}
\usepackage{lipsum}
\usepackage{epsfig}

\theoremstyle{plain}
\newtheorem{theorem}{Theorem}[section]
\newtheorem{lemma}[theorem]{Lemma}

\theoremstyle{definition}

\theoremstyle{remark}

\newcommand{\done}[1]{{\color{brown}{#1}}}

\newcommand{\nobeta}{{\textit{No Beta}}}
\newcommand{\betacutslb}{{\textit{Beta Lb}}}
\newcommand{\betacutsub}{{\textit{Beta Ub}}}
\newcommand{\betacuts}{{\textit{Beta Eq}}}
\newcommand{\allbeta}{{\textit{All Beta}}}
\newcommand{\extrabeta}{{\textit{Extra Beta}}}
\newcommand{\algdefault}{{\textit{Default}}}
\newcommand{\algnopre}{{\textit{No Preprocess}}}
\newcommand{\algnoinit}{{\textit{No Init}}}

\begin{document}

\begin{frontmatter}



\title{{An improved} column-generation-based matheuristic for learning classification trees}

\author[a]{Krunal Kishor Patel}
\author[b]{Guy Desaulniers}
\author[a,c]{Andrea Lodi}

\affiliation[a]{organization={CERC, Polytechnique Montr\'eal},
            addressline={2500 Chemin de Polytechnique}, 
            city={Montr\'eal},
            postcode={H3T 1J4}, 
            state={QC},
            country={Canada}}

\affiliation[b]{organization={Polytechnique Montr\'eal and GERAD},
            addressline={2500 Chemin de Polytechnique}, 
            city={Montr\'eal},
            postcode={H3T 1J4}, 
            state={QC},
            country={Canada}}

\affiliation[c]{organization={Jacobs Technion-Cornell Institute, Cornell Tech and Technion - IIT},
            addressline={2 West Loop Road}, 
            city={New York},
            postcode={10044}, 
            state={NY},
            country={USA}}

\begin{abstract}

Decision trees are highly interpretable models for solving classification problems in machine learning (ML). The standard ML algorithms for training decision trees are fast but generate suboptimal trees in terms of accuracy. Other discrete optimization models in the literature address the optimality problem but only work well on relatively small datasets. \cite{firat2020column} proposed a column-generation-based heuristic approach for learning decision trees. This approach improves scalability and can work with large datasets. In this paper, we describe improvements to this column generation approach. First, we modify the subproblem model to significantly reduce the number of subproblems in multiclass classification instances. Next, we show that the data-dependent constraints in the master problem are implied, and use them as cutting planes. Furthermore, we describe a separation model to generate data points for which the linear programming relaxation solution violates their corresponding constraints. We conclude by presenting computational results that show that these modifications result in better scalability.

\end{abstract}


\begin{highlights}
\item Improved subproblem model resulting in fewer subproblems than in \cite{firat2020column}.
\item Using data-dependent constraints as cutting planes in the master problem.
\item An optimization model to generate these cutting planes on demand.
\item A preprocessing and initialization routine that results in faster training.
\end{highlights}


\end{frontmatter}


\section{Introduction}
\label{sec:intro}

In machine learning (ML), a classification problem consists in predicting, from a predefined set of classes, the class to which a data point belongs. Each data point (also called data row) is described by features that are used to predict its class. Classification problems can be solved using decision trees that are highly interpretable models. In a decision tree, each internal node contains a test based on the dataset's features. In this work, we focus on univariate binary decision trees. In such trees, the internal node tests, called hereafter split checks, use only a single feature that can vary from one node to another. Each internal node has two branches. Each leaf node is associated with a target class. A data row starts at the tree's root node and follows the branches based on its feature values and the split checks at the internal nodes. Finally, it reaches a leaf node, where it is classified in the target class associated with that leaf node. 

To determine the split checks of a decision tree, a supervised learning algorithm is run on a training dataset, where the real class of each data point is known. Usually, the goal of this algorithm is to maximize accuracy, i.e., the number of training data points correctly classified. After the training is completed, the decision tree with its selected split checks can be used to classify unseen data rows. The standard ML algorithms for training decision trees, like CART \citep{breiman1984cart} and ID3 \citep{quinlan1986id3}, use heuristics that optimize for one-depth level accuracies. They are fast in terms of training times. However, the generated trees are suboptimal in terms of overall accuracy.

Learning optimal binary decision trees is an NP-complete problem \citep{laurent1976dtreenpc}. Many authors have created optimization models for learning decision trees optimal in terms of accuracy. \cite{gunluk2021dtreemip}, \cite{verwer2017dtreemip}, \cite{verwer2019dtreemip}, \cite{bertsimas2017dtreemip}, and \cite{aghaei2021dtreemaxflow} introduced Mixed Integer Linear Programs (MIPs) for training decision trees. {For a recent survey on continuous and MIP models for training decision trees, refer to \cite{carrizosa2021mathematical}. \cite{d2024margin} described a mixed integer quadratic model for learning optimal trees. \cite{blanco2022robust,blanco2023multiclass} introduced mixed integer non-linear models for training decision trees. \cite{blanquero2021optimal} proposed a continuous optimization approach for training randomized decision trees. \cite{nijssen2007dl8} introduced a dynamic-programming-based exact algorithm to compute optimal decision trees.} \cite{narodytska2018dtreesat} presented a SAT model to train decision trees. Finally, \cite{verhaeghe2020dtreecp} proposed a constraint programming model for training decision trees for binary classification tasks. 

These models can learn trees with better accuracies compared to CART and ID3. Sometimes they can find the optimal decision tree for a given dataset. However, the biggest issue with these models is scalability. \cite{firat2020column} stated that \cite{bertsimas2017dtreemip} and \cite{verwer2017dtreemip} MIP models failed to handle datasets with more than 10,000 rows. Except for the binary classification model of \cite{verhaeghe2020dtreecp}, no other paper (from the ones cited above) showed results on datasets with more than 10,000 rows. 

To address the scalability problem, \cite{firat2020column} introduced a column-generation-based heuristic to train more accurate univariate binary decision trees of predefined depth. In this approach, the master problem uses one variable for each path in the tree, where a path is defined by a set of nodes from the root node to a leaf node, associated with split checks, and a target class at the leaf node. This results in a large number of variables that are generated using subproblems. \cite{firat2020column} proposed a subproblem model that requires solving one subproblem for each leaf and each possible target class. {Integer solutions are obtained by transforming the restricted master problem of the last column generation iteration into a MIP and solving it using a MIP solver.} 

{It should be noted that this approach would still be heuristic even if a branch-and-price algorithm was used because not all possible split checks are considered at each node. However, in our experiments, we analyze the optimality of the computed solution for a fixed set of candidate split checks} to gain information about the best accuracy that can be achieved using branch-and-price. 

In this paper, we improve the \cite{firat2020column} column generation approach to make it faster and more scalable. This includes a preprocessing and initialization routine that reduces the size of the master problem and subproblems in the column generation process and results in faster training. Our main contributions are as follows:

\begin{enumerate}
    \item We introduce a modified subproblem model for the column generation approach with more variables and constraints than the one proposed in \cite{firat2020column} but requires fewer subproblems overall (one subproblem per leaf instead of one subproblem per leaf and target class). We show that the updated subproblem model results in faster training.
    \item We prove that the data-dependent constraints in the master problem {(enforcing that each row in a given dataset reaches a single leaf in the tree)}  are implied and can, thus, be added to the model as cutting planes only if they are violated. 
    \item {We extend the set of these constraints to unknown rows by providing a separation model that can generate any unlabeled row for which the corresponding constraint is  violated.}
    \item {We computationally show the significant benefits of our approach with respect to \cite{firat2020column} on datasets from the literature.}
\end{enumerate}


This paper is organized as follows. Section \ref{sec:firatoverview} presents an overview of the column-generation-based heuristic for training decision trees proposed in \cite{firat2020column}. In Section \ref{sec:methods}, we describe the proposed modifications to this matheuristic. In Section \ref{sec:results}, we report the results of various computational experiments to show the effect of our modifications. Finally, in Section \ref{sec:conclusion}, we present our conclusions and some future research directions.

\section{Overview of \cite{firat2020column} column generation approach}
\label{sec:firatoverview}

In this section, we describe the column-generation-based heuristic proposed in \cite{firat2020column} for generating decision trees. The authors focus on creating univariate binary decision trees of a fixed depth. The trees can handle multiclass classification, unlike in \cite{gunluk2021dtreemip} and \cite{verhaeghe2020dtreecp}. As this is a heuristic, the authors do not claim to generate optimal classification trees. Instead, they try to address the scalability issues in the other optimization models for learning decision trees.
	
The problem that \cite{firat2020column} attempt to solve is as follows. Consider a training dataset with rows in set $R$, features in set $F$, and targets in set $T$. Without loss of generality, assume that each feature is numerical (if not, we can convert it into a numerical feature by using natural numbers or one-hot encoding). The goal is to learn a complete binary classification tree of a given depth that maximizes accuracy on the training dataset. In other words, we want to find the best split checks for each internal node and the best target values for each leaf node in the decision tree. We use $N_{int}$ and $N_{lf}$ to refer to the set of all internal nodes and the set of all leaf nodes, respectively.
	
Each split check $a$ has two components. A feature $f_a \in F$ and a threshold value $\mu_a \in \mathbb{R}$ for that feature. A row $r \in R$ with value $v_r^{f_a}$ for feature $f_a$ takes the left branch from the node with this split check if $v_r^{f_a} \leq \mu_a$ and takes the right branch otherwise. 
	
One can consider optimizing over all possible split checks at each node. However, this can be computationally very expensive. Instead, \cite{firat2020column} used a subset $S_j$ of all possible split checks (called candidate split checks) for each node $j \in N_{int}$. Because of this, the learned tree may not be optimal. Nevertheless, such a tree provides better accuracy than the trees generated, for example, by CART.
	
To generate the candidate split checks, \cite{firat2020column} use the `Threshold sampling' process. They used 300 runs of the CART algorithm on randomly selected 90\% training data rows. For each run, they collect the generated split check for each node. From the collected split checks, they select the most frequent $q$ split checks for each node $j$ and add them to $S_j$, where $q = \left\lfloor \frac{150}{|N_{int}|}\right\rfloor$ for the root node and $q = \left\lfloor \frac{100}{|N_{int}|}\right\rfloor$ for the other nodes. Finally, they run CART on the entire training dataset. They select all the split checks generated in this run, i.e., they add them to their respective set $S_j$. Considering these split checks ensures that the final output will have accuracy at least as high as the CART output if the associated model is solved to optimality.
	
The goal of the learning problem is to find a split check for each node $j \in N_{int}$ from the given set $S_j$ and a target from $T$ for each leaf node to achieve maximum accuracy on the training dataset. 

\cite{firat2020column} proposed to model this problem as the following MIP, called the integer master problem (integer MP) in the context of a column generation algorithm. Consider the paths from the root node to the leaf nodes in the tree. Each path has a specific split check assigned to each internal node it contains and a specific target assigned to its leaf node. For each path, we define a binary variable that takes value 1 if the path is selected and 0 otherwise. The MIP 
ensures that the selected paths are consistent with each other, i.e., they form a tree. Relying on the notation presented in Table~\ref{tab:integermpnotations}, the MIP is as follows: 

\begin{table}[t]
\caption{Notation for the MIP (\ref{masterold}).}
\label{tab:integermpnotations}
\begin{center}
\scalebox{0.85}{
\begin{tabular}{lp{0.7\textwidth}}
    \textbf{Sets} & \\
    $R$ & set of rows in the dataset. \\
    $F$ & set of features in the dataset. \\
    $N_{lf}, N_{int}$& sets of leaf and internal nodes in the decision tree. \\
    $p_{BT}(l)$ & set of nodes in the paths to leaf $l$ in binary tree.\\
    $DP_l$& set of decision paths ending in leaf $l$.\\
    $R^{l}(p)$& subset of rows directed to leaf $l$ through path $p$.\\
    $S_j$& set of candidate split checks for node $j$.\\*[3pt]
    \textbf{Parameters} & \\
    $s_p(j)$ &split check assigned at node $j$ in path $p$.\\
    $CP(p)$& number of correct predictions  for path $p$.\\*[3pt]
    \textbf{Decision Variables}& \\
    $x_p$& binary variable indicating if path $p \in DP_l$ is assigned to leaf $l \in N_{lf}$.\\
    $\rho_{j,a}$& binary variable indicating if split check $a \in S_j$ is assigned to node $j \in N_{int}$.
\end{tabular}
}
\end{center}
\end{table}

\begin{subequations}
\label{masterold}
\begin{align}
    Max \quad & \sum_{l\in N_{lf}} \sum_{p \in DP_{l}} CP(p)x_p \label{masteroldobj}\\
    s.t.  \quad &  \sum_{p \in DP_{l}} x_p = 1,  \quad \forall l \in  N_{lf} \label{alpha}\\
    & \sum_{l\in N_{lf}} \sum_{\substack{p \in DP_{l}:\\r\in R^l(p)}} x_p = 1, \quad \forall r \in R \label{beta}\\
    & \sum_{\substack{p \in DP_{l}: \\s_p(j) = a}} x_p = \rho_{j,a}, \quad \forall l \in N_{lf}, j \in p_{BT}(l) \cap N_{int}, a \in S_j \label{gamma}\\
    & x_p \in \{0,1\},  \quad \forall p \in DP_{l}, l \in N_{lf} \label{firatm2}\\
    & \rho_{j,a} \in \{0,1\}, \quad \forall j \in N_{int}, a \in S_j. \label{firatbinrho}
\end{align}
\end{subequations}

The objective function (\ref{masteroldobj}) maximizes accuracy (minimizes the misclassification error). Constraints (\ref{alpha}) ensure that exactly one path is selected for each leaf. Constraints (\ref{beta}) force each row to follow exactly one selected path. The known class of the row may not match the target of the followed path. Constraints (\ref{gamma}) impose that the selected paths are consistent with respect to split checks on common nodes: if two paths are selected with a common node, the same split check must be assigned to the common node in both paths. Note that in this model, the variable upper bounds are implied and can be dropped while solving the linear relaxation of the MIP. 

In general, each row in the training dataset has equal weights. However, this model is able to support rows with different weights. In this case, we need to compute the objective coefficients $CP(p)$ as a weighted sum of the correctly classified rows for each path $p$.

Model (\ref{masterold}) contains a large number of variables, one per path in sets $DP_l$, $l \in N_{lf}$. As proposed by \cite{firat2020column}, we can alleviate this drawback by applying column generation. In this context, we refer to the linear relaxation of the MIP as the master problem (MP) and to the MIP itself as the integer MP. Column generation is used to solve the MP. At each iteration, it solves using a standard linear programming solver a restricted MP (RMP), i.e., the MP restricted to a small subset of its variables. In the first iteration, the RMP is initialized with the paths generated from the last run of CART on the complete training dataset in the threshold sampling process. Solving the current RMP provides a pair of optimal primal and dual solutions. To verify if the primal solution is also optimal for the whole MP, a set of subproblems (SPs), namely, one for each leaf node $l\in N_{lf}$ and each target class $t\in T$, is solved with the goal of identifying  columns (paths) with a positive reduced cost with respect to the current RMP dual solution. When such columns are found, they are added to the RMP and another iteration is started. Otherwise, the column generation process stops with an optimal solution to the MP.

Relying on the additional notation presented in Table \ref{tab:suboldnotations}, the SP for leaf $l \in N_{lf}$ and target $t \in T$ can be formulated as the following MIP:
	
\begin{table}[t]
\caption{Additional notation for the SP (\ref{subold}) defined for leaf $l$ and target $t$.}
\label{tab:suboldnotations}
\begin{center}
\scalebox{0.85}{
    \begin{tabular}{lp{0.7\textwidth}}
        \textbf{Sets} & \\
        $R_t$ & set of rows in the dataset with target $t$. \\
        $LC(l)$ & set of nodes in $N_{int}$ that have left child in $p_{BT}(l)$.\\
        $RC(l)$ & set of nodes in $N_{int}$ that have right child in $p_{BT}(l)$.\\
        $T(r)$ &  set of split checks for which row $r$ takes the left branch: $\{a = (f_a , \mu_a ) \in \cup_{j\in N_{int}} S_{j}: v_r^{f_a} \leq \mu_a\} $.\\
        $F(r)$ &  set of split checks for which row $r$ takes the right branch: $\{a = (f_a , \mu_a ) \in \cup_{j\in N_{int}} S_{j}: v_r^{f_a} > \mu_a\} $.\\*[3pt]        
        \textbf{Parameters} & \\
        $k$& depth of the decision tree, levels are indexed by $h = 0,..., k - 1$.\\
        $v_r^f$& value of feature $f$ in row $r$.\\
        $\alpha_l$& dual value of constraint (\ref{alpha}) for leaf $l$.\\
        $\beta_r$& dual value of constraint (\ref{beta}) for row $r$.\\
        $\gamma_{l,j,a}$& dual value of constraint (\ref{gamma}) for leaf $l$, node $j$, and split check $a$.\\*[3pt]
        \textbf{Decision Variables}& \\
        $y_r$& binary variable indicating if row $r \in R$ reaches leaf $l$.\\
        $u_{j,a}$& binary variable indicating if split check $a \in S_j$ is assigned to node $j \in p_{BT}(l)$.
    \end{tabular}
    }
\end{center}
\end{table}
	
\begin{subequations}
    \label{subold}
   \allowdisplaybreaks
    \begin{align}
        \quad Max \quad & \sum_{r\in R_t} y_r - \alpha_l - \sum_{j \in p_{BT}(l)}\sum_{a \in S_j} \gamma_{l,j,a}u_{j,a} - \sum_{r\in R}\beta_r y_r \label{oldspobj}\\
        s.t.  \quad &  \sum_{a \in S_j} u_{j,a} = 1,  \quad \forall j \in  p_{BT}(l) \label{firats1}\\
        & y_r \leq \sum_{a\in S_j \cap T(r)} u_{j,a} , \quad \forall j \in LC(l), r\in R \label{firats2}\\
        & y_r \leq \sum_{a\in S_j \cap F(r)} u_{j,a} , \quad \forall j \in RC(l), r\in R \label{firats3}\\
        & \sum_{j\in LC(l)} \sum_{a \in S_j \cap T(r)} u_{j,a} +  \sum_{j\in RC(l)} \sum_{a \in S_j \cap F(r)} u_{j,a} - (k-1) \leq y_r, \quad \forall r \in R \label{firats4}\\
        & \sum_{j \in p_{BT}(l)} u_{j,a} \leq 1, \quad \forall a \in \bigcup_{j \in p_{BT}(l)} S_{j} \label{firats5}\\
        & y_r \in \{0,1\},  \quad \forall r \in R \\
        & u_{j,a} \in \{0,1\}, \quad \forall j \in RC(l) \cup LC(l), a \in S_j.
    \end{align}
\end{subequations}

The objective function (\ref{oldspobj}) aims at maximizing the reduced cost of the path generated. Constraints (\ref{firats1}) ensure that exactly one split is selected for each node in the path. Constraints (\ref{firats2})-(\ref{firats4}) evaluate if the row $r$ would reach the leaf by following the generated path. That is, the variable $y_r$ takes the value 1 if and only if a split check from the set $T(r)$ is selected for all nodes in $LC(l)$ and a split check from the set $F(r)$ is selected for all the nodes in $RC(l)$. Because of the objective direction, constraint (\ref{firats4}) for a row~$r$ is only needed if {$\beta_r \geq 1$ for $r \in R_t$ and $\beta_r \geq 0$ for $r \not\in R_t$}. Constraints (\ref{firats5}) impose that each split is selected at most once in a path. Given that these constraints led to infeasible SPs when the candidate sets of split checks are small for some nodes in our  experiments, we have decided to remove them. {This only increases the search space but does not change the rest of the algorithm}.

Because model (\ref{subold}) is a MIP, generating paths using it can be computationally expensive. To address that, \cite{firat2020column} also introduced a pricing heuristic. This heuristic randomly generates paths and evaluates their reduced costs. If the reduced cost of a generated path is positive, the path is added to the RMP. Model (\ref{subold}) is only used if the heuristic fails to find any path with a positive reduced cost.

In theory, the column generation process stops if model (\ref{subold}) proves that all columns have a non-positive reduced cost. Given that the convergence might be slow, it can also terminate when a time limit is reached.

To obtain an optimal integer solution to the integer MP (\ref{masterold}), non-truncated column generation can be embedded into a branch-and-bound algorithm to yield a branch-and-price algorithm \citep{Barnhart1998}. However, to limit the computational times, \cite{firat2020column} rather applied a RMP heuristic \citep{JoncourEtAl2010,SadykovEtAl2019} that consists in converting the last RMP solved into a MIP that can be solved using a commercial MIP solver without adding any new columns. The selected paths in the final solution describe a valid learned tree.

\section{Modifications of the column generation approach}
\label{sec:methods}

This section describes our modifications of the column-generation-based heuristic of \cite{firat2020column}. In Section \ref{sec:mergedsp}, we show how to reduce the number of SPs by adding the target computation for the leaf nodes in the constraints of the SPs. In Section \ref{sec:masterrow}, we analyze the constraints (\ref{beta}) in the integer MP (\ref{masterold}) and show that they are implied by the other constraints. In Section \ref{sec:sepmasterrow}, we describe a separation algorithm for using constraints (\ref{beta}) as cutting planes to speed up the solving process. 
	
\subsection{Merged SPs}
\label{sec:mergedsp}

In \cite{firat2020column}, there is one SP (\ref{subold}) for each leaf $l \in N_{lf}$ and each target $t \in T$. Consequently, there are $|N_{lf}| \times |T|$ SPs in total. We propose to merge the SPs by incorporating the target computation in the MIP model using extra variables and constraints. This reduces the number of SPs to $|N_{lf}|$. This modification is inspired by the flow-based MIP formulation of \cite{aghaei2021dtreemaxflow}, which is very similar to (\ref{subold}), except that it focuses on the entire tree instead of just one path. Relying on the additional notation presented in Table \ref{subnew}, the merged SP for a given leaf $l \in N_{lf}$ is formulated as follows:
	
\begin{table}[t]
\caption{Additional notation for the SP (\ref{subnew}) for leaf $l$.}
\label{tab:subnewnotations}
\begin{center}
\scalebox{0.85}{
    \begin{tabular}{lp{0.7\textwidth}}
       
        \textbf{Parameters} & \\
        $W_r$ & weight of row $r$.\\*[3pt]
        
        \textbf{Decision Variables}& \\
        $z_r$& binary variable indicating if row $r \in R$ reaches leaf $l$ and has the same target as the path being generated.\\
        $w_t$& binary variable indicating if target $t \in T$ is selected for the generated path.\\
    \end{tabular}
}
\end{center}
\end{table}

\begin{subequations}
    \label{subnew}
   \allowdisplaybreaks
    \begin{align}
        \quad Max \quad & \sum_{r\in R} W_r z_r - \alpha_l - \sum_{j \in p_{BT}(l)}\sum_{a \in S_j} \gamma_{l,j,a}u_{j,a} - \sum_{r\in R}\beta_r y_r \\
        s.t.  \quad &  \sum_{a \in S_j} u_{j,a} = 1,  \quad \forall j \in  p_{BT}(l) \label{newsp1}\\
        & y_r \leq \sum_{a\in S_j \cap T(r)} u_{j,a} , \quad \forall j \in LC(l), r\in R \label{newsp2}\\
        & y_r \leq \sum_{a\in S_j \cap F(r)} u_{j,a} , \quad \forall j \in RC(l), r\in R \label{newsp3}\\
        & \sum_{j\in LC(l)} \sum_{a \in S_j \cap T(r)} u_{j,a} +  \sum_{j\in RC(l)} \sum_{a \in S_j \cap F(r)} u_{j,a} - (k-1) \leq y_r, \quad \forall r \in R \label{newsp4}\\
        & z_r \leq y_r, \quad \forall r\in R  \label{newsp6}\\
        & z_r \leq w_t, \quad \forall t\in T, r \in R_t  \label{newsp7}\\
        & \sum_{t \in T} w_t = 1  \label{newsp8}\\
        & z_r \in \{0,1\},  \quad \forall r \in R \\
        & y_r \in \{0,1\},  \quad \forall r \in R \\
        & u_{j,a} \in \{0,1\} \quad \forall j \in RC(l) \cup LC(l), a \in S_j.
    \end{align}
\end{subequations} 

Except for the first term, the objective function is the same as (\ref{oldspobj}). The new first term $\sum_{r \in R} W_rz_r$ computes the weighted sum of the number of rows being correctly classified. The weight $W_r$ of a row $r$ is generally set to 1. However, using the weights in the model allows us to train on the dataset where the rows have different weights (i.e., not all the rows are equally important). We also need this because our approach changes the weights of the rows in the dataset during the preprocessing stage. 

Constraints (\ref{newsp1})-(\ref{newsp4}) are the same as constraints (\ref{firats1})-(\ref{firats4}). As previously mentioned, constraint (\ref{newsp4}) for a row $r$ is only useful if {$\beta_r \geq 1$ for $r \in R_t$ and $\beta_r \geq 0$ for $r \not\in R_t$} because of the objective direction. Constraints (\ref{newsp6})-(\ref{newsp7}) ensure that the variable $z_r$ takes value 1 if the row $r$ reaches leaf $l$ and has the same target as the path being generated. Constraint (\ref{newsp8}) imposes the selection of exactly one target for the generated path. 

Overall the new SP model (\ref{subnew}) has $2 \times |R| + 1$ new constraints and $|R| + |T|$ new variables. Because of the merging, we have fewer SPs, but the model is larger than model (\ref{subold}).

\subsection{Redundancy of {MP} constraints (\ref{beta})}
\label{sec:masterrow}

Observe that constraints (\ref{beta}) in the integer MP (\ref{masterold}) are dependent on the dataset but do not depend on the target of the rows. We show that these constraints are implied by constraints (\ref{alpha}), (\ref{gamma})--(\ref{firatbinrho}). The proof of this result relies on the the following lemma that is presented first.


\begin{lemma}
\label{lma:singlesplit}
    In any solution to the model (\ref{masterold}) that satisfies the constraints (\ref{alpha}), (\ref{gamma}), and (\ref{firatbinrho}), there exists a split check $a^*_j\in S_j$ for every internal node $j\in N_{int}$ such that  
\begin{equation}
     \rho_{j,a^*_j} = 1 \quad \text{and} \quad \rho_{j,a} = 0, \quad \forall a \in S_j \setminus \{ a^*_j \}. \label{delta}
\end{equation}
\end{lemma}

%

\begin{proof}
Consider an internal node $ j \in N_{int}$.
There exists a leaf $l \in N_{lf}$ such that $ j \in p_{BT}(l)$.
We sum constraints (\ref{gamma}) over all candidate split checks $a \in S_j$ for leaf $l$ and node $j$ to obtain
$$\sum_{a \in S_j}\rho_{j,a} = \sum_{a \in S_j}\sum_{\substack{p \in DP_{l}: \\s(j) = a}} x_p = \sum_{p \in DP_{l}} x_p = 1,$$
where the last equality follows from constraints (\ref{alpha}). Given that the $\rho_{j,a}$ variables are binary according to (\ref{firatbinrho}), there is a single variable in the first sum, say for $a = a^*_j$, that takes value 1. All the others are equal to 0.
\end{proof}

\begin{theorem}
\label{thm:impliedbeta}
    For any decision tree with depth $k \geq 1$, constraints (\ref{beta}) are implied by constraints (\ref{alpha}), (\ref{gamma})--(\ref{firatbinrho}). In other words, model (\ref{masterold}) is equivalent to:

    \begin{subequations}
    \label{masternew}
    \begin{align}
        Max \quad & \sum_{l\in N_{lf}} \sum_{p \in DP_{l}} CP(p)x_p \\
        s.t.  \quad &  \sum_{p \in DP_{l}} x_p = 1,  \quad \forall l \in  N_{lf} \label{newalpha}\\
        & \sum_{\substack{p \in DP_{l}: \\s_p(j) = a}} x_p = \rho_{j,a}, \quad \forall l \in N_{lf}, j \in p_{BT}(l) \cap N_{int}, a \in S_j \label{newgamma}\\
        & x_p \in \{0,1\}, \quad \forall p \in DP_{l}, l \in N_{lf} \label{newfiratm2}\\
        & \rho_{j,a} \in \{0,1\} \quad \forall j \in N_{int}, a \in S_j.
    \end{align}
    \end{subequations}
\end{theorem}
\begin{proof}

Since model (\ref{masternew}) is a relaxation of model (\ref{masterold}), any feasible solution to model (\ref{masterold}) is also feasible for model (\ref{masternew}). We now show the opposite direction.

Consider an arbitrary data row $r \in R$. Let $P_r$ be the subset of paths that are followed by row $r$ from root node to some leaf node (which may have a different target than the row). We need to show that for this row $r$, any feasible solution to model (\ref{masternew}) contains exactly one path in set $P_r$.

Consider a feasible solution to model (\ref{masternew}), which, therefore, satisfies constraints (\ref{alpha}), (\ref{gamma}), and (\ref{firatbinrho}). By Lemma \ref{lma:singlesplit}, we know that, in this solution, exactly one split check is assigned to each node. At the root node of the decision tree, let the assigned split check be $a = (f_a, \mu_a)$. Clearly, either $v_r^{f_a} \leq \mu_a$ and the row follows the left branch, or $v_r^{f_a} > \mu_a$ and the row follows the right branch. Using a similar argument at each other node reached by the row, we infer that row $r$ visits a unique set of nodes (and split checks) to reach a unique leaf node $l \in N_{lf}$ in the tree. 

In set $P_r$, there are exactly $|T|$ different paths that correspond to these sets of nodes and split checks, namely, one for each possible target in $T$. However, constraints (\ref{newalpha}) and (\ref{newfiratm2}) ensure that exactly one of these paths is selected for leaf $l$, implying that constraint (\ref{beta}) for row $r$ is satisfied by this solution of model (\ref{masternew}). Consequently, all feasible solutions to model (\ref{masternew}) are also feasible for model (\ref{masterold}).
\end{proof}

\subsection{Search for violated MP constraints (\ref{beta})}
\label{sec:sepmasterrow}

Although the constraints (\ref{beta}) are implied in the integer MP (\ref{masterold}), they are needed for a stronger linear relaxation (see results in Section \ref{sec:resmasterrow}). So, instead of removing them, we use them as cutting planes. In other words, we solve the RMP without constraints (\ref{beta}) and add to the RMP only those that are violated by the linear relaxation solution. The search for the violated constraints is performed by inspection of the individual rows in $R$.

Furthermore, we experimented with considering only one bound (lower or upper) when adding constraints (\ref{beta}) as cutting planes. We found that the upper bound inequality version of constraints (\ref{beta}) performs similarly to using them as equality constraints. On the other hand, the lower bound inequality version of constraints (\ref{beta}) results in a poor linear relaxation. Section \ref{sec:resmasterrow} will detail experimental results.

The linear relaxation strength of model  (\ref{masterold}) depends on the dataset $R$. If it does not contain sufficient rows to obtain a strong bound, the column-generation-based heuristic may result in a poor-quality solution. Observe, however, that there is no need to limit the constraints (\ref{beta}) to a predetermined dataset $R$. Indeed, any possible row described by a set of feature values should follow a single path in the tree to reach a  leaf (independently of the target associated with this leaf). Consequently, we do not limit our search for violated constraints (\ref{beta}) to the rows in set $R$, but also to any potential row for which we do not know the target. We will call the latter rows the \textit{unlabeled rows} as no target will be associated with them.

To generate an unlabeled row for which a given MP solution violates the corresponding constraint (\ref{beta}), we develop a separation model. In fact, this model does not identify a specific row but determines whether the row would take the left or right branch on each split check. With {this} information, we can deduce bounds on the feature values that a row should have to yield a violated constraint and construct a corresponding unlabeled row by selecting arbitrary feature values respecting these bounds.

Consider the solution of the current RMP. Let $DP$ denote the subset of paths $p$ such that $x_p$ takes a positive value $x^*_p$ in this solution. Furthermore, for any row $r$, let $P_r$ be the subset of paths in $DP$ followed by this row. We want to generate an unlabeled row $r$ (or equivalently, find the branch taken on each split check) such that $\sum_{p \in P_r} x^*_p \neq 1$. This constraint can be violated in two ways. Either the upper bound or the lower bound can be violated. From preliminary experiments using the existing constraints (\ref{beta}) with only the lower or only the upper bound, we observed that using only the upper bound results in a stronger linear relaxation. Hence, we only focus on the violations of the upper bound. 

Let $\theta_p$ be a binary variable that takes value $1$ if the generated row follows path $p\in DP$. Let $\psi_a$ be a binary variable that takes value $1$ if the row takes the right branch on split check $a\in S$ and $0$ otherwise. Thus, the row follows a path $p$ if $\psi_a = 1$ for all split checks $a\in R_p$ and $\psi_a = 0$ for all split checks $a\in L_p$, where $R_p$ (resp. $L_p$) is the set of split checks for which path $p$ takes the right (resp. left) branch.

With these variable definitions, we can determine if the generated row follows a path $p\in DP$ using the following equation:
\begin{equation} 
\theta_p = \bigwedge_{a \in L_p} \neg \psi_a \land \bigwedge_{a \in R_p} \psi_a.
\end{equation}

As different split checks on possibly different paths may involve the same feature, there must be some consistency between the branches taken by the unlabeled row on these checks, i.e., there might be some implied relations between split checks. Assume that two split checks $a$ and $b$ are defined on the same feature $f$ with different threshold values $\mu_a$ and $\mu_b$ with $\mu_a \leq \mu_b$. In that case, whenever the row takes the left branch on split check $a$ ($v_r^{f} \leq \mu_a$ or, equivalently, $\psi_a = 0$), it must also take the left branch on split check $b$ ($v_r^{f} \leq \mu_b$ or $\psi_b = 0$). This relation translates into the constraint
\begin{equation}
 \neg \psi_a \implies \neg \psi_b.
\end{equation}

Relying on the notation presented in Table \ref{tab:cutgennotations}, the separation model for generating violated constraints (\ref{beta}) for a given solution of the MP can be expressed as follows:

\begin{table}[t]
\caption{Notation for the separation model (\ref{cut_gen}).}
\label{tab:cutgennotations}
\begin{center}
\scalebox{0.85}{
    \begin{tabular}{lp{0.7\textwidth}}
        \textbf{Sets} & \\
        $DP$& set of {paths $p$ with $x^*_p > 0$}.\\
        $S$& set of all split checks.\\
        $L_p$ & set of split checks where path $p$ takes the left branch.\\
        $R_p$ & set of split checks where path $p$ takes the right branch.\\
        \textbf{Parameters} & \\
        $x_p^{*}$ & value of variable $x_p$ in the current RMP solution.\\
        $f_a$ & feature of split check $a$.\\
        $\mu_a$ & threshold of split check $a$.\\
        \textbf{Decision Variables}& \\
        $\theta_p$& binary variable indicating if the generated row follows path $p \in DP$.\\
        $\psi_a$& binary variable indicating if the generated row takes the right branch on split check $a \in S.$
    \end{tabular}
    }
\end{center}
\end{table}

\begin{subequations}
    \label{cut_gen}
    \begin{align}
        Max \quad & \sum_{p \in DP}x_p^* \theta_p \label{sep_obj1}\\
        s.t.  \quad &  \theta_p = \bigwedge_{a \in L_p} \neg \psi_a \land \bigwedge_{a \in R_p} \psi_a, \quad \forall p \in DP \label{path_split_link} \\
        & \neg \psi_a \implies \neg \psi_b, \quad \forall a,b \in S: f_a = f_b, \mu_a \leq \mu_b \label{split_consistency}\\
        & \theta_p \in \{0,1\},  \quad \forall p \in DP\\
        & \psi_a \in \{0,1\},  \quad \forall a \in S. \label{bin_psi}
    \end{align}
\end{subequations}


The objective function (\ref{sep_obj1}) aims at finding an unlabeled row with a maximum left-hand side in constraint (\ref{beta}). Constraints (\ref{path_split_link}) and (\ref{split_consistency}) have been introduced above. 

Model (\ref{cut_gen}) is solved by constraint programming, using the CP-SAT solver. A violated constraint (\ref{beta}) is found whenever the optimal value of (\ref{cut_gen}) is strictly greater than 1. In fact, we retrieve all intermediate solutions obtained by the solver that have an objective value greater than 1 and generate one cut for each of them. These cuts, if any, are then added to the MP.

\section{Computational results}
\label{sec:results}
In this section, we report the results obtained in our computational experiments. Section \ref{sec:expsetup} gives an overview of the datasets used for these experiments and the experimental setup. In Section \ref{sec:resinitprep}, we describe our preprocessing and initialization routines and show their effects. In Section \ref{sec:resmergedsp}, we report the computational results on merging the SPs as described in Section \ref{sec:mergedsp}. In Section \ref{sec:resmasterrow}, we present the computational results for different ways of using constraints (\ref{beta}) in the MP. Finally, in Section \ref{sec:resall}, we compare our revised column generation matheuristic with that of \cite{firat2020column}.

\subsection{Datasets and experimental setup}
\label{sec:expsetup}

We used 12 datasets from the UCI repository \citep{UCI2019} for our computational experiments. There are six \emph{small} datasets involving between 500 and 10,000 rows and six \emph{large} datasets containing over 10,000 rows. The small datasets are also used in \cite{verwer2019dtreemip} and \cite{firat2020column}. These datasets are already processed, and all features are numerical with no missing values. Each dataset is split into training (50\% of the rows) and testing (25\%) parts, and the split is done randomly five times.\footnote{The remaining 25\% was used by \cite{bertsimas2017dtreemip} for validation but, {like \cite{firat2020column},}  we did not use it.} The reported performance for each dataset is averaged over these five train-test splits. We use the same train-test splits as \cite{verwer2019dtreemip} and \cite{firat2020column}. 

For the large datasets, the train-test split used by \cite{firat2020column} is not available. We performed the same cleaning steps as in \cite{firat2020column}, including transforming classes to integers and converting string features into numerical ones. For each dataset, we then generated five random train-test splits like for the small datasets. The dataset specifications are listed in Table \ref{tab:datasets}. The datasets and the source code are available at https://github.com/krooonal/col\_gen\_estimator/tree/dtreedev.

\begin{table}[h]
\caption{Dataset {sizes}.}
\label{tab:datasets}
\begin{center}
\scalebox{0.85}{
\begin{tabular}{l|c|c|c}
Dataset           & Number of rows & Number of features & Number of classes\\
\hline
\multicolumn{4}{l}{Small} \\
\hline
Tic tac toe       & 958            & 18                 & 2                 \\
Indian diabetes  & 768            & 8                  & 2                 \\
Car evaluation    & 1728           & 5                  & 4                 \\
Seismic bumps     & 2584           & 18                 & 2                 \\
Spambase          & 4601           & 57                 & 2                 \\
Statlog satellite & 4435           & 36                 & 6                 \\
\hline
\multicolumn{4}{l}{Large} \\ 
\hline
Default Credit    & 30000          & 23                 & 2                 \\
Hand posture      & 78095          & 33                 & 5                 \\
HTRU 2            & 17898          & 8                  & 2                 \\
Letter Recognition      & 20000          & 16                 & 26                \\
Magic04           & 19020          & 10                 & 2                 \\
Shuttle           & 43500          & 9                  & 7                 \\
\hline
\end{tabular}
}
\end{center}
\end{table}

For the experiments, we used a cluster of CPU servers, each with two sockets of Intel(R) Xeon(R) Gold 6258R CPU @ 2.70GHz, 28 cores each (total of 56 cores), and 512GB RAM. However, for training, each process was limited to 8 threads and 16 GB of memory. We imposed these limits to ensure a fair comparison with \cite{firat2020column}. {For the comparison, we use the numbers reported in \cite{firat2020column}.}

We used Python 3 for all computations, scikit-learn version 1.2.2 \citep{scikit-learn} for running CART, Gurobi 9.5.0 \citep{gurobi} to solve the RMP (with or without integrality requirements), and the CP-SAT solver from Google OR-Tools \citep{Google-OR-Tools} to solve the SPs ((\ref{subold}) or (\ref{subnew})) and to generate cutting planes by solving (\ref{cut_gen}). Indeed, preliminary experiments showed that using the CP-SAT solver for the SPs is about twice as fast as using Gurobi. Since the CP-SAT solver can only solve problems with integer coefficients, we multiply the objective coefficients by a scaling factor $10^5$ and round them to the nearest integer. {This is equivalent to solving the model using a MIP solver with $10^{-5}$ as the optimality threshold.}

Our computational approach is the same as in \cite{firat2020column} {except for the optional} preprocessing and initialization steps \ref{extra_init} and \ref{preprocessing} below. The preprocessing step is further described in Section \ref{sec:resinitprep}. The complete process is as follows:

\begin{enumerate}
    \item \label{step_initCART} Run 300 iterations of CART on randomly selected 90\% of the training data and collect split checks for each node. 
    \item Select the $q$ most frequent split checks for each internal node as the candidate split checks, where $q = \left\lfloor \frac{150}{|N_{int}|}\right\rfloor$ for the root node and $q = \left\lfloor \frac{100}{|N_{int}|}\right\rfloor$ for the other nodes.
    \item \label{step_allCART} Run CART on the entire training data. Add the generated split checks to the candidate split checks and the generated paths to the initial RMP. This step ensures that the final output tree will have accuracy at least as high as the CART output.
    \item \label{extra_init} [Optional] Run 100 iterations of CART on randomly selected 80\% of the training data and add the generated paths to the initial RMP. A path is added only if the split checks it contains belong to the candidate split checks of the corresponding nodes.
    \item \label{preprocessing} [Optional] Preprocessing: If any two rows in the dataset have the same target and the rows take the same branches on all candidate split checks of the reachable nodes, remove one of the rows and increase the weight of the other row by one.
    \item Perform column generation iterations to solve the MP. Stop if it is solved to optimality (all SPs  failed to generate new columns) or a time limit is reached.
    \item Solve the integer MP (\ref{masterold}) restricted to the generated columns to get the final decision tree.
\end{enumerate}

{In our revised column generation approach, we apply a simpler version of the pricing heuristic presented in \cite{firat2020column}. We start with randomly selecting a leaf $l \in N_{lf}$, then for each node $j \in p_{BT}(l)$, we randomly select a split check $a \in S_j$ that has not been selected before. If we are able to select different split checks for each node in the path, we compute the target of the path that maximizes the accuracy for the set of rows following this path. Finally, we compute the reduced cost of the generated path using the duals from model (\ref{masterold}). If the reduced cost is greater than the threshold value $10^{-6}$, we add the path to the RMP. We repeat this process 100 times in each column generation iteration. Unlike \cite{firat2020column}, we do not maintain a set of feasible columns that are not yet added to the RMP.}

Except for the experiments in Section \ref{sec:resall}, the depth of the tree is fixed to $k = 4$ for all our experiments.

\subsection{Initialization and preprocessing results}
\label{sec:resinitprep}

In steps \ref{step_initCART}--\ref{step_allCART} above, we follow the process of generating candidate split checks as in \cite{firat2020column}. To initialize the RMP, \cite{firat2020column} only retain the paths generated from the last run of CART on the entire dataset (step \ref{step_allCART}). We extend further this initialization by collecting all paths generated during additional iterations of CART on a randomly selected subset of the training dataset (step \ref{extra_init}).

After generating all the candidate split checks and paths for initialization, we compute the set of reachable nodes for each row in the training dataset. A node $j$ is reachable by a row $r$ if an assignment of split checks exists for the ancestor nodes of node $j$ such that row $r$ can reach node $j$ by following the branches of the ancestor nodes. We use these sets of reachable nodes to compare pairs of rows in the training dataset.
	
If two rows take the same branch on all the split checks on the reachable nodes, for one of the rows, we remove the corresponding constraint (\ref{beta}) from integer MP (\ref{masterold}) and constraints (\ref{newsp2})-(\ref{newsp4}) from the SP (\ref{subnew}) (or constraints (\ref{firats2})-(\ref{firats4}) from the SP (\ref{subold})). We then add 1 to the weight $W_r$ of the other row $r$ that is kept as it is. Removing such duplicate constraints reduce the size of integer MP (\ref{masterold}) and SPs (\ref{subnew}) (or SPs (\ref{subold})).

We trained the models with three variations of our approach to evaluate the effects of the preprocessing step \ref{preprocessing} and the initialization step \ref{extra_init}. For these experiments, we imposed no time limit. The variations are the following:

\renewcommand{\descriptionlabel}[1]{\hspace{\labelsep}\text{#1}}
\begin{description}
    \item[\algdefault:] Default training process with steps \ref{extra_init} and \ref{preprocessing} enabled.
    \item[\algnopre:] Training process with preprocessing in step \ref{preprocessing} disabled.
    \item[\algnoinit:] Training process with extra initialization in step \ref{extra_init} disabled.
\end{description}

Figure \ref{fig:prep_notl} shows the solving time for all variations.\footnote{Result for the Letter Recognition instances is not included as the training process did not finish after 15 hours.} We can observe that \algnopre\ gives the best solving times for the small datasets (left part of the black divider) except for the Car evaluation dataset. However, for large datasets, it gives the worst solving times except for Magic04 dataset. This suggests that the preprocessing step \ref{preprocessing} is helpful only on large instances. 

The solving times for the \algnoinit\ variant are always larger than the \algdefault\ variant except for Spambase and Shuttle datasets. This suggests that extra initialization is helpful in general.

All variations have similar accuracy gains over CART, as expected {(see Figure \ref{fig:prep_notl_acc})}. The average training accuracy gain over CART is 0.9\% for all variations.

\begin{figure}[t]
\centering
\includegraphics[scale=0.6]{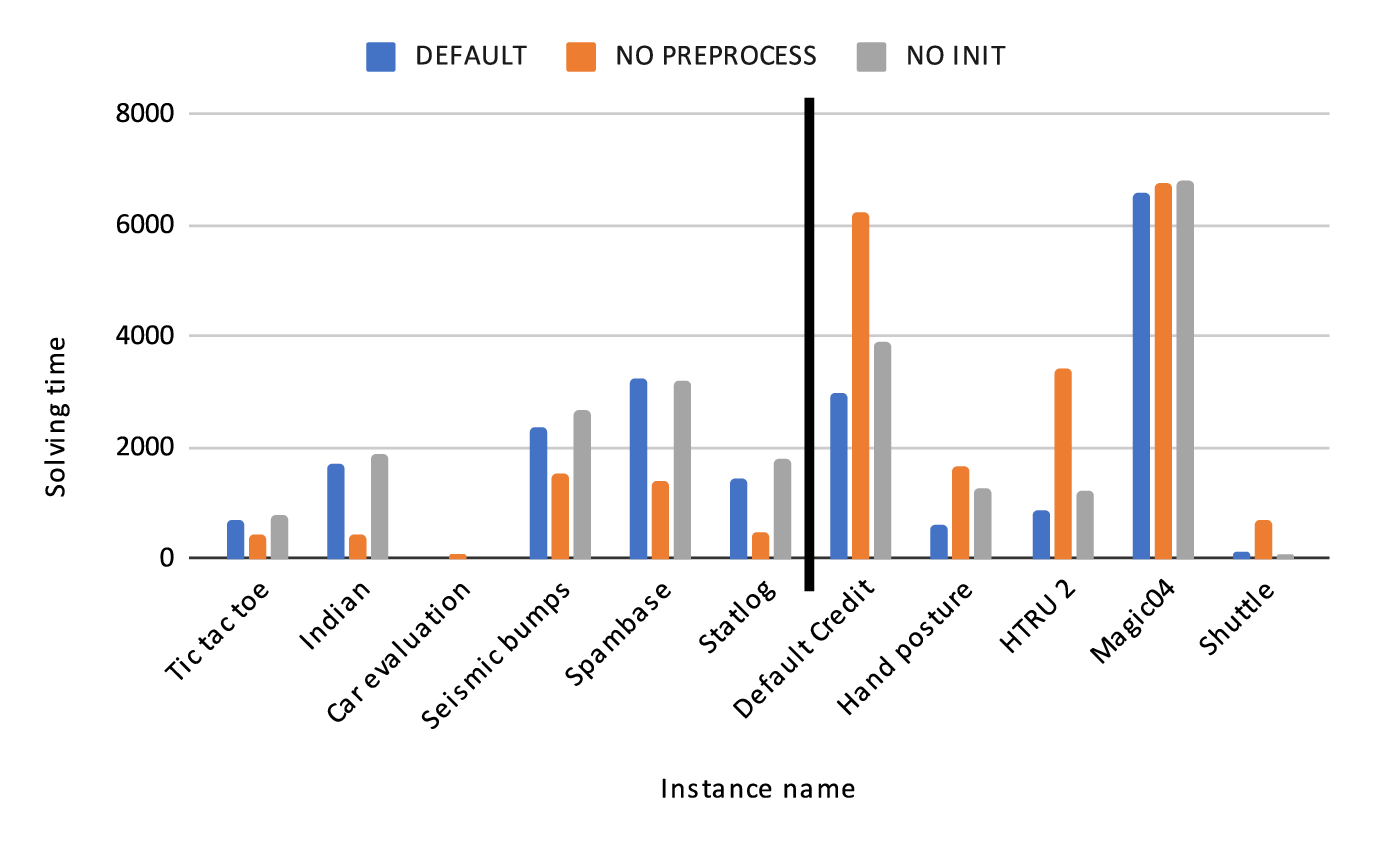}
\caption{Solving time for different ways of using preprocessing and initialization.}
\label{fig:prep_notl}
\end{figure}

\begin{figure}[t]
\centering
\includegraphics[scale=1]{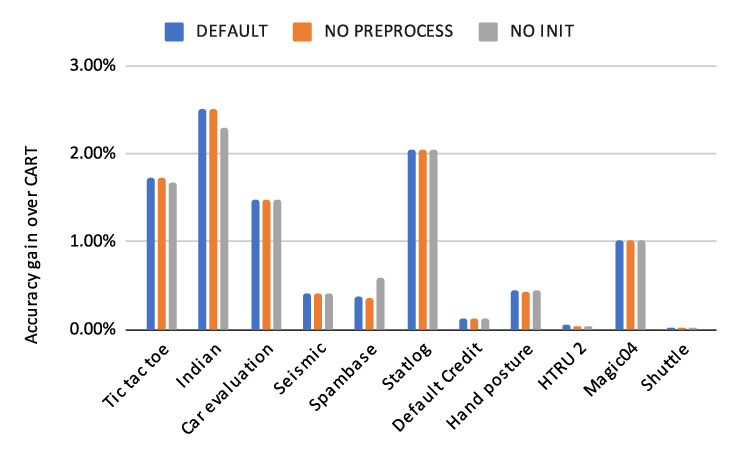}
\caption{{Accuracy gain over CART for different ways of using preprocessing and initialization.}}
\label{fig:prep_notl_acc}
\end{figure}

\subsection{Merged SPs results}
\label{sec:resmergedsp}

To evaluate the effect of merging the SPs, we disabled the pricing heuristic. We set a time limit of 1 hour (without the pricing heuristic, the solving times are too large). The extra initialization and preprocessing steps \ref{extra_init} and~\ref{preprocessing} were enabled for all datasets.

\begin{figure}[p]
\centering
\includegraphics[scale=0.8]{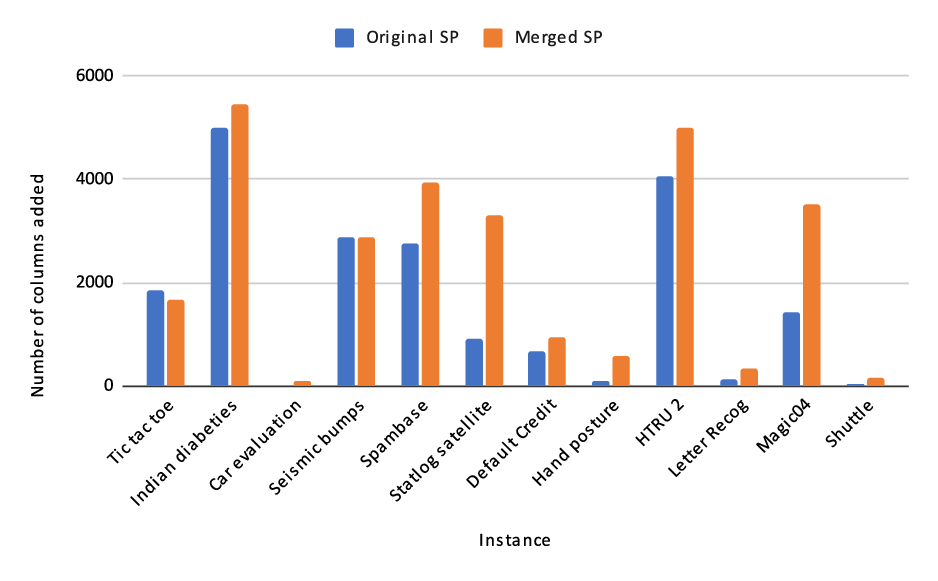}
\caption{Number of columns added with original and merged SPs.}
\label{fig:old_vs_new_sp_cols}
\end{figure}

\begin{figure}[p]
\centering
\includegraphics[scale=0.8]{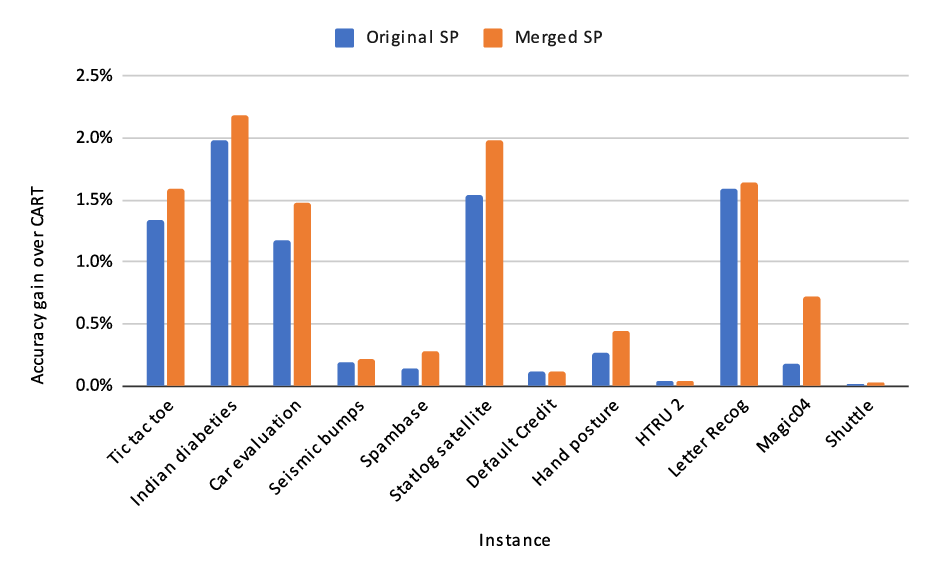}
\caption{Accuracy gain over CART on training datasets with original and merged SPs.}
\label{fig:old_vs_new_sp}
\end{figure}

The merged SPs are generally faster than the original SPs and hence add more columns to the RMP in the same time limit. We can observe this in Figure \ref{fig:old_vs_new_sp_cols}. Furthermore, as shown in Figure \ref{fig:old_vs_new_sp}, the extra columns added because of the faster SPs lead to larger gains in accuracy over CART on the training dataset.  The dataset `Tic-tac-toe' is the only notable exception where the original SPs added more columns compared to the merged SPs. However, the accuracy gains over CART are larger for the approach with the merged SPs. 

\subsection{MP constraints (\ref{beta}) results}
\label{sec:resmasterrow}

To evaluate the effect of constraints (\ref{beta}), also referred to in the following as the \textit{beta cuts}, we experimented with the following six algorithm variants:
\begin{description}
    \item[\nobeta:] The MP is solved without using any constraints (\ref{beta}).
    \item[\betacutslb:] The MP does not contain any constraints (\ref{beta}) at the beginning, but they are added as inequality cuts with only the lower bound. 
    \item[\betacutsub:] The MP does not contain any constraints (\ref{beta}) at the beginning, but they are added as inequality cuts with only the upper bound.
    \item[\betacuts:] The MP does not contain any constraints (\ref{beta}) at the beginning, but they are added as equality cuts.
    \item[\allbeta:] The MP contains all constraints (\ref{beta}) from the beginning. This is the setting used by \cite{firat2020column}.
    \item[\extrabeta:] Same as the \betacuts\ variant except that additional equality cuts for unlabeled rows can be generated using model (\ref{cut_gen}).
\end{description}

Consequently, beta cuts associated with unlabeled rows are only generated in the \extrabeta\ variant. 
In all variants with cuts, the cut generation algorithm is invoked at every 10 iterations of column generation. For these experiments, we considered no time limit and enabled the initialization and preprocessing steps \ref{extra_init} and  \ref{preprocessing}. 

Figure \ref{fig:beta_notl} shows the solving times for all variants.\footnote{Results for the Letter Recognition dataset are not included as the training process did not finish after 15 hours.} The solving times for \nobeta\ are the shortest among all variants. The \betacutslb\ variant has similar (but slightly larger) solving times than \nobeta. However, these two variants provide the worst linear relaxation {value} compared to the others (see {MP Value in} Figure \ref{fig:beta_notl_lprelax}). Hence, they also result in the lowest accuracy gains over CART compared to the other variants (see Figure \ref{fig:beta_notl_acc}).

\begin{figure}[t]
\centering
\includegraphics[scale=0.8]{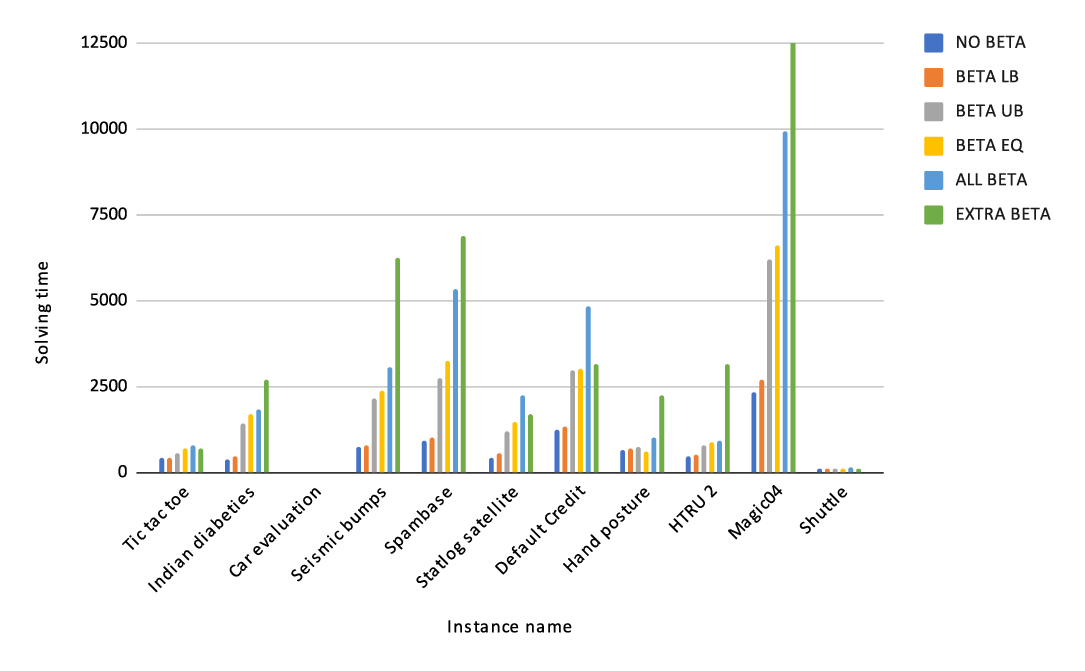}
\caption{Solving time for different ways of using constraints (\ref{beta}).}
\label{fig:beta_notl}
\end{figure}

The solving times for \betacutsub\ are larger than those of \nobeta. The \betacuts\ variant has similar but slightly larger solving times compared to the \betacutsub\ variant. These two variants are consistently faster than the \allbeta\ variant. Finally, the variant \extrabeta\ has the largest solving times. These four variants (\betacutsub, \betacuts, \allbeta, and \extrabeta) have similar linear relaxation (Figure \ref{fig:beta_notl_lprelax}) and result in similar accuracy gains over CART on the training dataset (Figure \ref{fig:beta_notl_acc}). The \extrabeta\ variant has slightly better accuracy gains over CART compared to the other three variants. However, the difference is too small to observe it in the diagrams (on average 0.95\% versus 0.93\% for \betacutsub, \betacuts, and \allbeta\ variants). 

\begin{figure}[t]
\centering
\includegraphics[scale=0.7]{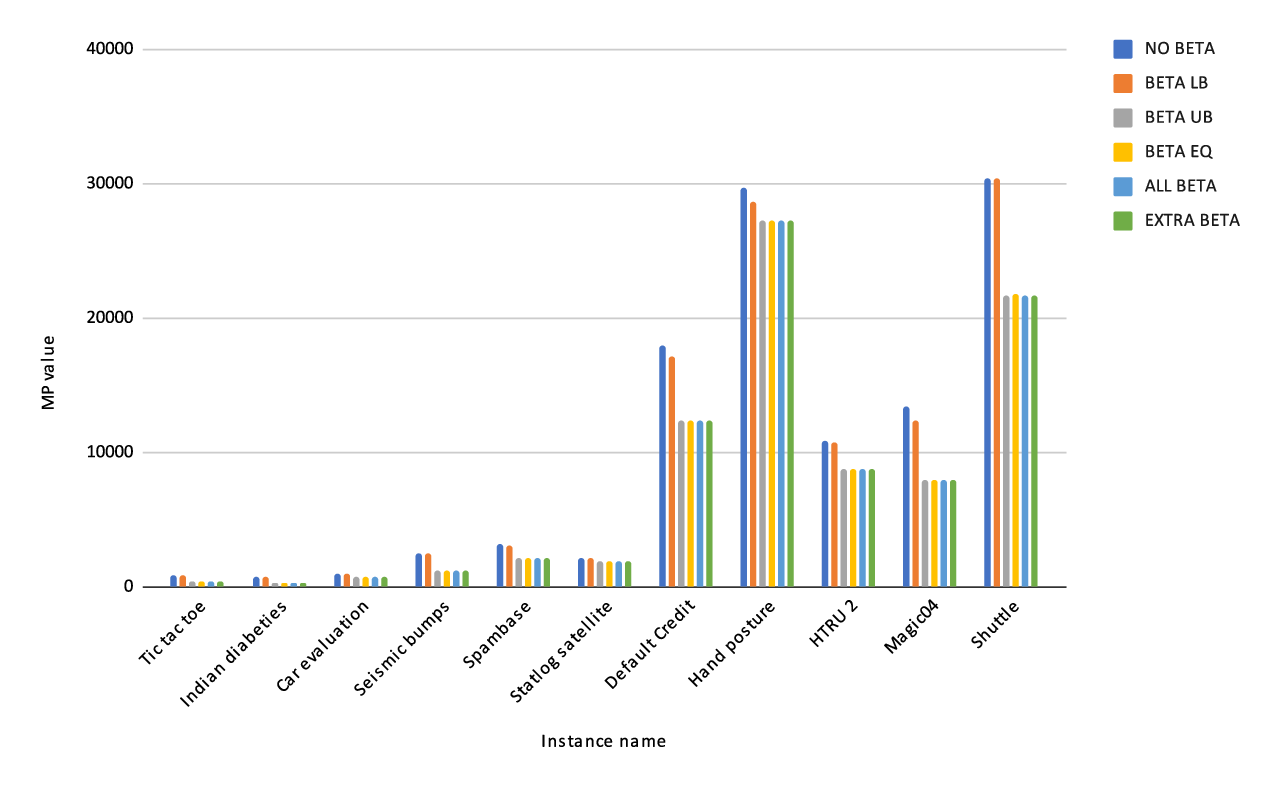}
\caption{MP value for different ways of using constraints (\ref{beta}).}
\label{fig:beta_notl_lprelax}
\end{figure}

\begin{figure}[t]
\centering
\includegraphics[scale=0.8]{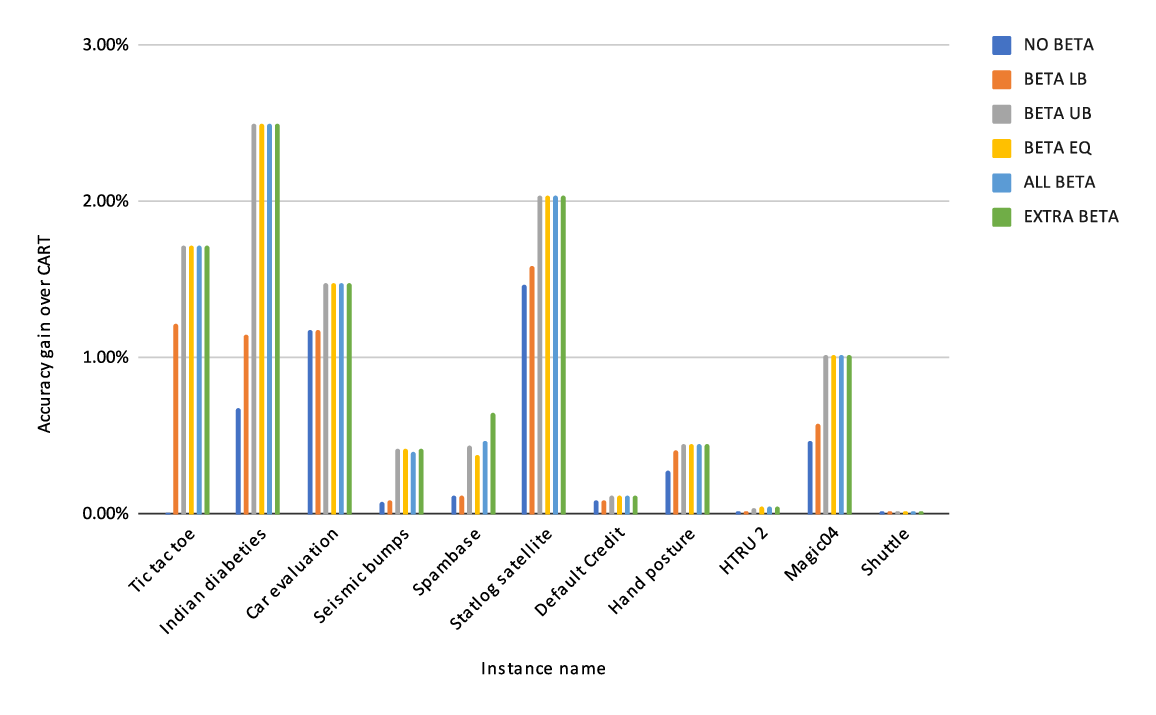}
\caption{Training accuracy gain over CART for different ways of using constraints (\ref{beta}).}
\label{fig:beta_notl_acc}
\end{figure}

Table \ref{tab:beta_smalltl_acc_gain} shows the average training accuracy gains over CART when we use a time limit of 600 seconds. The gains for \betacutsub, \betacuts, \allbeta, and \extrabeta\ are comparable and much better than those achieved by the other two variants. Given that the average solving time of \betacutsub\ is less than the three others, we conclude that \betacutsub\ has the best trade-off between solving time and training accuracy gain over CART. 

\begin{table}[H]
\caption{Average training accuracy gain over CART for different ways of using constraints (\ref{beta}) with a 600s time limit.}
\label{tab:beta_smalltl_acc_gain}
\begin{center}
\scalebox{0.85}{
\begin{tabular}{l|c|c}
Method & Train accuracy (\%) & Gain over CART (\%) \\
\hline
CART & 81.40 & 0.00 \\
\nobeta & 81.85 & 0.45 \\
\betacutslb & 82.04 & 0.64 \\
\betacutsub & 82.29 & 0.90 \\
\betacuts & 82.29 & 0.89 \\
\allbeta & 82.28 & 0.87 \\
\extrabeta & 82.26 & 0.86
\end{tabular}
}
\end{center}
\end{table}

Our approach does not consider all possible split checks for each node. Hence, we cannot prove the optimality of the generated tree even if we would be applying a full branch-and-price algorithm. However, for the problem limited to the candidate split checks for each internal node, we can analyze the optimality of our solutions. We compared the MP optimal values to the values of the best integer solutions found. The \extrabeta\ variant can prove optimality for 53 instances out of 55 instances (see Figure \ref{fig:beta_notl_optimality}). Thus, in most cases, we can prove optimality using a heuristic approach. This suggests that branch-and-price might not improve solution quality significantly in the presence of the beta cuts generated using model (\ref{cut_gen}). Note that the solution  values produced by the other variants (\betacutsub, \betacuts, and \allbeta) are very close to the optimal values, as shown in Figure \ref{fig:beta_notl_acc}.

\begin{figure}[t]
\centering
\includegraphics[scale=0.8]{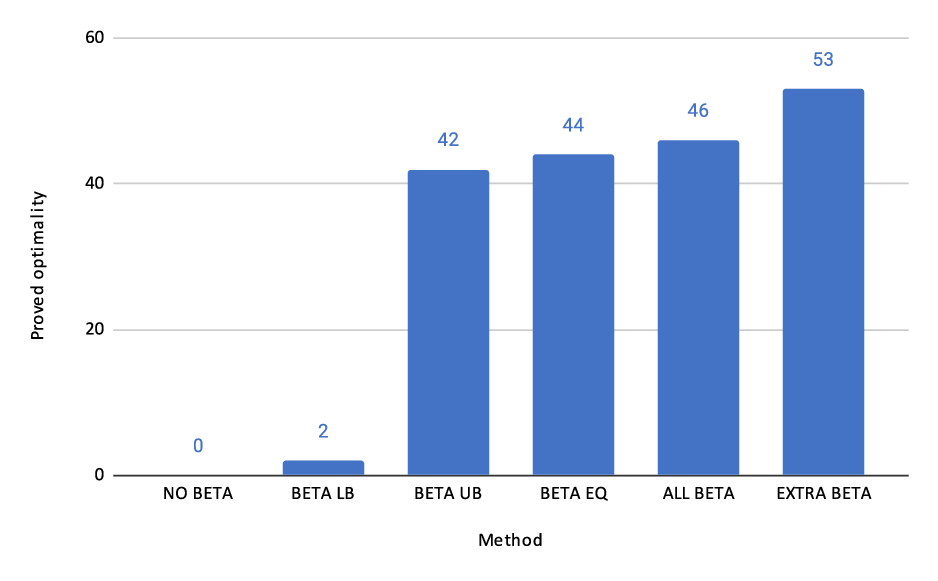}
\caption{Number of instances solved to optimality for different ways of using constraints (\ref{beta}).}
\label{fig:beta_notl_optimality}
\end{figure}

\subsection{Comparison with \cite{firat2020column}}
\label{sec:resall}

This section compares our best results against the  results reported in \cite{firat2020column}. For these experiments, we used the merged SP model (\ref{subnew}). We disabled the preprocessing step \ref{preprocessing} for small datasets and enabled it for the large datasets, while the initialization step \ref{extra_init} was enabled for all datasets. As done by \cite{firat2020column}, we set a time limit of 600 seconds for the training and considered three tree depths $k=2, 3, 4$.

Table \ref{tab:smallcomparision} compares training accuracy gains over CART against the column generation approach of \cite{firat2020column} where CG stands for column generation. Better gains are highlighted in bold. The revised column generation approach achieved strictly larger gains on 13 cases and strictly lower gains on 2 cases out of a total of 18 comparisons.

\begin{table}[t]
\caption{Comparision of training accuracy gains for small datasets.}
\label{tab:smallcomparision}
\begin{center}
\scalebox{0.85}{
\begin{tabular}{l|c|p{0.08\textwidth}|p{0.08\textwidth}|p{0.08\textwidth}|p{0.08\textwidth}|p{0.08\textwidth}|p{0.08\textwidth}}
\multicolumn{2}{c|}{} & \multicolumn{6}{c}{Accuracy (\%)} \\
\cline{3-8}
\multicolumn{2}{c|}{} & \multicolumn{3}{c|}{Firat et al.\ CG} & \multicolumn{3}{c}{Revised CG} \\
\hline
        Instance& $k$ & CART & CG & Gain & CART & CG & Gain \\
\hline
Tic tac toe & 2 & 71.2 & 71.8 & \textbf{0.6} & 71.3 & 71.8 & 0.5 \\
 & 3 & 75.4 & 76.7 & 1.3 & 75.4 & 77.4 & \textbf{1.9} \\
 & 4 & 84.4 & 85.4 & 1.0 & 84.5 & 86.2 & \textbf{1.8} \\
Indian diabeties & 2 & 77.3 & 78.8 & \textbf{1.5} & 77.3 & 78.8 & \textbf{1.5} \\
 & 3 & 78.9 & 81.2 & 2.3 & 78.9 & 81.3 & \textbf{2.4} \\
 & 4 & 82.9 & 84.2 & 1.3 & 82.9 & 85.3 & \textbf{2.4} \\
Car evaluation & 2 & 76.9 & 76.9 & \textbf{0.0} & 76.9 & 76.9 & \textbf{0.0} \\
 & 3 & 79.0 & 79.8 & 0.8 & 79.0 & 80.2 & \textbf{1.2} \\
 & 4 & 84.2 & 85.2 & 1.0 & 84.2 & 85.7 & \textbf{1.5} \\
Seismic bumps & 2 & 93.1 & 93.3 & 0.2 & 93.1 & 93.4 & \textbf{0.3} \\
 & 3 & 93.4 & 93.7 & \textbf{0.3} & 93.4 & 93.7 & \textbf{0.3} \\
 & 4 & 93.9 & 94.2 & 0.3 & 93.9 & 94.3 & \textbf{0.4} \\
Spambase & 2 & 86.0 & 87.1 & 1.1 & 85.9 & 87.2 & \textbf{1.3} \\
 & 3 & 89.6 & 90.3 & \textbf{0.7} & 89.6 & 90.3 & 0.6 \\
 & 4 & 91.6 & 91.6 & 0.0 & 91.6 & 92.0 & \textbf{0.4} \\
Statlog satellite & 2 & 63.2 & 64.0 & 0.8 & 63.2 & 64.3 & \textbf{1.1} \\
 & 3 & 78.7 & 79.5 & 0.8 & 78.7 & 80.0 & \textbf{1.3} \\
 & 4 & 81.6 & 82.9 & 1.3 & 81.6 & 83.6 & \textbf{2.0} \\
 \hline
\end{tabular}
}
\end{center}
\end{table}

The accuracies for training are not available for the large datasets in \cite{firat2020column}. We compare the testing accuracy gains against the column generation approach of \cite{firat2020column} in Table \ref{tab:largecomparision}. The revised column generation approach achieves strictly larger gains on 6 cases and strictly lower gains on 3 cases out of 18 instances. We also report the training accuracy gains obtained by our column generation heuristic for the large datasets in Table \ref{tab:largetraining}. Out of 18 instances, our heuristic yields a gain on 11 instances, with a maximum gain of 3.2\%.

\begin{table}[t]
\caption{Comparision of testing accuracy gains for the large datasets.}
\label{tab:largecomparision}
\begin{center} 
\scalebox{0.85}{
\begin{tabular}{l|c|p{0.08\textwidth}|p{0.08\textwidth}|p{0.08\textwidth}|p{0.08\textwidth}|p{0.08\textwidth}|p{0.08\textwidth}}
\multicolumn{2}{c|}{} & \multicolumn{6}{c}{Accuracy (\%)} \\
\cline{3-8}
\multicolumn{2}{c|}{} & \multicolumn{3}{c|}{Firat et al.\ CG} & \multicolumn{3}{c}{Revised CG} \\
\hline
        Instance& $k$ & CART & CG & Gain & CART & CG & Gain \\
\hline
Default Credit & 2 & 82.3 & 82.3 & \textbf{0.0} & 81.9 & 81.9 & \textbf{0.0} \\
 & 3 & 82.3 & 82.3 & \textbf{0.0} & 82.0 & 82.0 & \textbf{0.0} \\
 & 4 & 82.3 & 82.3 & \textbf{0.0} & 81.9 & 82.0 & \textbf{0.0} \\
Hand posture & 2 & 56.4 & 56.4 & \textbf{0.0} & 56.6 & 56.6 & \textbf{0.0} \\
 & 3 & 62.5 & 62.8 & 0.3 & 62.6 & 63.0 & \textbf{0.4} \\
 & 4 & 69.0 & 69.1 & 0.1 & 69.4 & 69.7 & \textbf{0.3} \\
HTRU 2 & 2 & 97.8 & 97.8 & \textbf{0.0} & 97.6 & 97.6 & \textbf{0.0} \\
 & 3 & 97.9 & 97.9 & \textbf{0.0} & 97.7 & 97.7 & \textbf{0.0} \\
 & 4 & 98.0 & 98.0 & \textbf{0.0} & 97.8 & 97.7 & -0.1 \\
Letter Recog & 2 & 12.5 & 12.7 & 0.2 & 12.3 & 12.7 & \textbf{0.4} \\
 & 3 & 17.7 & 18.6 & 0.9 & 17.6 & 19.6 & \textbf{2.0} \\
 & 4 & 24.8 & 27.0 & 2.2 & 24.5 & 27.8 & \textbf{3.3} \\
Magic04 & 2 & 78.4 & 79.1 & \textbf{0.7} & 79.0 & 79.7 & \textbf{0.7} \\
 & 3 & 79.1 & 80.1 & \textbf{1.0} & 79.1 & 79.9 & 0.8 \\
 & 4 & 81.5 & 81.5 & 0.0 & 81.6 & 82.4 & \textbf{0.8} \\
Shuttle & 2 & 93.7 & 93.7 & \textbf{0.0} & 93.9 & 93.9 & \textbf{0.0} \\
 & 3 & 99.6 & 99.7 & \textbf{0.1} & 99.6 & 99.7 & 0.0 \\
 & 4 & 99.8 & 99.8 & \textbf{0.0} & 99.8 & 99.8 & \textbf{0.0} \\
 \hline
\end{tabular}
}
\end{center}
\end{table}

\begin{table}[t]
\caption{Training accuracy gains {of the revised CG heuristic} for the large datasets.}
\label{tab:largetraining}
\begin{center} 
\scalebox{0.85}{
\begin{tabular}{l|c|r|r|r}
\multicolumn{2}{c|}{} & \multicolumn{3}{c}{Accuracy (\%)} \\
\hline
Instance & $k$ & CART & CG  & Gain \\
\hline
Default Credit & 2 & 82.0 & 82.0 & 0.0 \\
 & 3 & 82.2 & 82.2 & 0.0 \\
 & 4 & 82.3 & 82.5 & 0.1 \\
Hand posture & 2 & 56.6 & 56.6 & 0.0 \\
 & 3 & 62.7 & 63.1 & 0.5 \\
 & 4 & 69.4 & 69.9 & 0.4 \\
HTRU 2 & 2 & 97.8 & 97.8 & 0.1 \\
 & 3 & 98.0 & 98.1 & 0.1 \\
 & 4 & 98.2 & 98.3 & 0.0 \\
Letter Recog & 2 & 13.1 & 13.4 & 0.3 \\
 & 3 & 18.2 & 20.8 & 2.6 \\
 & 4 & 25.6 & 28.8 & 3.2 \\
Magic04 & 2 & 79.6 & 80.0 & 0.4 \\
 & 3 & 80.0 & 81.0 & 1.0 \\
 & 4 & 82.7 & 83.7 & 1.0 \\
Shuttle & 2 & 93.9 & 93.9 & 0.0 \\
 & 3 & 99.7 & 99.7 & 0.0 \\
 & 4 & 99.8 & 99.9 & 0.0 \\
\hline
\end{tabular}
}
\end{center}
\end{table}

Note that both approaches focus on improving the training accuracies and do not take extra steps to generalize the performance across the testing datasets. However, even if the model is not trained {having generalization in mind}, it still can significantly improve over CART. With more focus on generalizing the performance over unseen data, we {should be able to produce} even better results. We discuss some directions to address this issue in {the next section.}

\section{Conclusion and future work}
\label{sec:conclusion}

In this work, we presented modifications to the column-generation-based heuristic of \cite{firat2020column} that can be applied to generate decision trees. First, we reduced the number of SPs by moving the target computation in the constraints. This results in a faster generation of columns. Then, we showed that the data-dependent constraints in the integer MP are implied and presented ways to use them as cutting planes. This helps solve the RMP faster. Furthermore, we described an optimization model to generate these cutting planes on demand even if the corresponding data row is not in the training dataset. These additional cutting planes helped to show that we can generate optimal decision trees for the given candidate split checks in most instances. {Note that this model can also be linearized into a MIP model, {an avenue that might be explored in future research}.} Finally, we described a process for better initializing the RMP and preprocessing the dataset to reduce the size of the model. The extra initialization and preprocessing steps further help to reduce the solving times, {especially for the larger datasets.}

As future work, we can consider developing a diving heuristic \citep[see][]{SadykovEtAl2019} to derive better integer solutions for the variant where constraints~(\ref{beta}) are not used (\nobeta). This variant has the best solving time but the worst solutions compared to the other variants. Also, as the SPs only change in the objective function across the column generation iterations, we can explore how the information generated by solving previous SPs can be exploited to solve the subsequent SPs faster \citep{MIPCC23}.

Finally, as we did not focus on generalizing the performance of our approach for out-of-sample datasets, the suggestions made by \cite{firat2020column} to generalize this performance can be studied, namely, to penalize the number of active leaves (i.e., reached by at least one row) in the integer MP and to enforce in the SP a minimum number of rows following any generated path. {Like minimizing the number of internal nodes in \cite{bertsimas2017dtreemip}, minimizing the number of active leaves aims at reducing tree complexity.} It would be interesting to study the performance of the proposed heuristic with these changes.

\vspace*{10mm}
\noindent \textbf{Declaration of interest:} None. 

\vspace*{5mm}
\noindent \textbf{CRediT authorship contribution statement} \\

\textbf{Krunal K. Patel:} Conceptualization, Methodology, Software, Writing -- original draft. \textbf{Guy Desaulniers:} Methodology, Validation, Writing -- review \& editing, Supervision.  \textbf{Andrea Lodi:} Methodology, Validation, Writing -- review \& editing, Supervision. 

\newpage
\noindent \textbf{Acknowledgements} \\

The first author is supported by a research grant from the project DEpendable and Explainable Learning (DEEL). We would like to thank Giuliano Antoniol for coordinating between DEEL and Polytechnique Montreal. 

\vspace*{5mm}





\bibliographystyle{elsarticle-harv}
\bibliography{references}
\end{document}